\documentclass[journal,11pt,twocolumn,final,]{IEEEtran}
\usepackage{cite}
\usepackage{amsmath,amssymb,amsfonts}
\usepackage{algorithmic}
\usepackage{graphicx}
\usepackage{textcomp}
\usepackage{xcolor}
\def\BibTeX{{\rm B\kern-.05em{\sc i\kern-.025em b}\kern-.08em
    T\kern-.1667em\lower.7ex\hbox{E}\kern-.125emX}}

\usepackage{nicefrac}       

\usepackage{microtype}
\usepackage{xcolor}

\usepackage{booktabs} 
\usepackage{hyperref}

\usepackage{algorithm}
\usepackage{algorithmic}

\usepackage{amssymb}
\usepackage{mathtools}
\usepackage{amsthm}
\usepackage{subfig}
\usepackage{bbm}
\usepackage{dsfont}

\theoremstyle{plain}
\newtheorem{theorem}{Theorem}[section]

\newtheorem{lemma}[theorem]{Lemma}

\theoremstyle{definition}

\theoremstyle{remark}

\newcommand{\E}{\mathbb{E}}
\newcommand{\R}{\mathbb{R}}
\newcommand{\N}{\mathbb{N}}

\newcommand{\1}{\mathds{1}}

\newcommand{\cX}{\mathcal{X}}

\newcommand{\ip}[2] {\langle #1, #2 \rangle }

\newcommand{\cE}{\mathcal{E}}

\newcommand{\cJ}{\mathcal{J}}

\newcommand{\cO}{\mathcal{O}}

\newcommand{\cN}{\mathcal{N}}

\DeclareMathOperator*{\argmin}{arg\,min}

\begin{document}

\title{A Communication-Efficient Adaptive Algorithm for Federated Learning under Cumulative Regret}

\author{Sudeep Salgia, Tamir Gabay, Qing Zhao, Kobi Cohen
\thanks{S. Salgia and Q. Zhao, are with the School of Electrical and Computer Engineering, Cornell University, Ithaca, NY, 14850, USA. Emails: \{ss3827, qz16\}@cornell.edu. T. Gabay and K. Cohen are with the School of Electrical and Computer Engineering, Ben-Gurion University of the Negev, Israel. Emails: 
{tamirgab@post.bgu.ac.il, kobi.cohen10@gmail.com}}\\
\thanks{This work was supported in part by by the Government of Israel - Israeli Ministry of Defense, Mission to the USA.}
}
%
\maketitle
\begin{abstract}
We consider the problem of online stochastic optimization in a distributed setting with $M$ clients connected through a central server. We develop a distributed online learning algorithm that achieves order-optimal cumulative regret with low communication cost measured in the total number of bits transmitted over the entire learning horizon. This is in contrast to existing studies which focus on the offline measure of simple regret for learning efficiency. The holistic measure for communication cost also departs from the prevailing approach that \emph{separately} tackles the communication frequency and the number of bits in each communication round.  
\end{abstract}

\section{Introduction}
\label{sec:introduction}

We study the problem of first-order online stochastic convex optimization in a distributed setup where $M$ clients aim to collaboratively minimize an unknown convex function $f$ by using noisy gradient estimates at sequentially queried points in the domain, which is a known subset of $\R^d$. 

One key performance metric for stochastic optimization is the learning efficiency. In existing studies, the offline measure of simple regret is commonly adopted as the measure for learning efficiency. Specifically, simple regret is defined as $f(\hat{x}_T) - \min f$, where $\hat{x}_T$ is a point returned by the algorithm at the end of the learning horizon and corresponds to the best estimate of the minimizer learnt by the algorithm. Since simple regret only cares about sub-optimality gap of the learned minimizer, it does not reflect the performance of an algorithm in an online setting, where it is important to control the running sum of excessive loss in real time during the learning process. A more appropriate metric in online settings is cumulative regret, which provides a cumulative assessment of the learning throughout the algorithm as is often needed in online learning setups.

Another key performance metric for distributed learning is the communication cost. The overall communication cost of a distributed learning algorithm consists of two parts:  the frequency of communications and the size of the message in each communication round. These two components of the overall communication cost have been largely dealt with \emph{separately} in the literature. A more holistic approach is needed to minimize the overall communication cost.  

\subsection{Main Results}

In contrast to existing studies, we adopt the online measure of cumulative regret for learning efficiency and a holistic measure for communication efficiency. We develop a distributed online learning algorithm and show that it achieves the order-optimal cumulative regret of $\cO(\log(MT))$ while incurring a low communication cost of $\cO(d \log(MT))$ \emph{total bits} over the entire learning horizon. We conjecture this is the minimum order of communication required to achieve a sublinear cumulative regret order.

The proposed algorithm, referred to as Communication-Efficient Adaptive Learning (CEAL), consists of a decision strategy working in tandem with a communication protocol that ensures the order-optimal regret with low communication cost. The decision strategy is characterized by a novel norm estimation routine embedded into its design. This routine is designed with the goal of estimating the norm of an unknown vector to within a multiplicative factor by using its noisy samples. This not only allows the algorithm to use high quality estimates of gradient to take larger steps in the direction of the gradient but also \emph{adaptively} tunes the time between communication rounds to ensure low communication frequency. This implicitly ensures that in the initial stage when the iterates are far from the optimum, the algorithm moves away from them quickly with more frequent communication and in the later stages, as the iterates moves closer to the minimum, the algorithm can afford to spend more time on them and communicate less frequently. Moreover, this adaptivity is achieved without any tuning parameters or knowledge of function parameters, making this approach robust to unknown function parameters. This decision strategy is complemented with a communication protocol that ensures low communication frequencies and small message size via quantization and encoding.

\subsection{Related Work}

There is an extensive literature on distributed stochastic optimization~\cite{Dekel2012, BrendanMcMahan2017, Lee2017, Li2019, Liu2019FedBCD, Cen2020, Dinh2020, Pathak2020, Reddi2020, Woodworth2020, Woodworth2020heterogenous, Karimireddy2020mime, Karimireddy2020scaffold, Yu2020, Wang2021unifiedFramework, Li2022Soteria, Das2022, Zhao2022}. Representative algorithms include Local-SGD\footnote{We collectively refer to all algorithms are based on the concept of Local-SGD, including the ones that employ momentum/variance reduction, to belong to this representative family.} (also known as FedAvg) and Minibatch-SGD. It has been well-established that these algorithms achieve order-optimal simple regret of $\cO(1/MT)$ over a learning horizon of length $T$ by using a weighted sum of all the iterates with a set of specifically designed weights based on the problem parameters. However, these results on simple regret do not necessarily imply sublinear, let alone order-optimal, guarantees on the cumulative regret performance of these algorithms. On the other hand, CEAL achieves order-optimal cumulative regret, which is a finer measure of performance as compared to simple regret. Moreover, it can be shown that, using the convexity of the underlying function, the order-optimal cumulative regret of CEAL also implies an order-optimal (upto logarithmic factors) simple regret, resulting in significantly stronger performance guarantees over existing results. Furthermore, these guarantees are achieved by CEAL without the knowledge of specific problem parameters.

In terms of communication cost, there is a large body of work developing communication efficient algorithms by either reducing communication frequency~\cite{Stich2019, Khaled2019, Haddadpour2019, Yu2019, Spiridonoff2020, Yu2020, Haddadpour2021federated, Mishchenko2022} or reducing message size~\cite{Konecny2016quantization, Bernstein2018, Basu2019, Sun2019communication, Dai2019HSQ, Tang2019DoubleSqueeze, Abdi2020, Reisizadeh2020fedpaq, Li2021CS, Zheng2021, Honig2022DAdaQuant, Jhunjhunwala2021adaptive,  Mahmoudi2022ALAQ, Wang2022Communication} by employing techniques like quantization and sparsification. See~\cite{Tang2020communication, Pouriyeh2022} and~\cite{Zhao2022communication} for a detailed survey of such approaches. However, as mentioned earlier, most of these works treat the two components of communication cost separately and hence focus on reducing only one of these two components. Specifically, for Minibatch-SGD, the existing studies focus only on reducing communication frequency while allowing for high-precision message exchange. Our work deviates from existing works in its attempt to consider a holistic characterization of the communication cost.

\section{Problem Formulation}
\label{sec:problem_formulation}

We study the problem of online optimization in a distributed setting with $M$ clients and a single central server. The clients collaboratively work together to minimize an unknown function $f : \cX \to \R$, where $\cX \subset \R^d$ is a convex, compact set. The function $f$ is known to be $\alpha$-strongly convex and $\beta$-smooth. At each time instant $t$, each client $m$ chooses a point $x_t^m$, based on the decision strategy of the algorithm, and observes a noisy observation of gradient at $x_t^m$, denoted by $G(x_t^m) = \nabla f(x_t^m) + \xi_t^m$. $\{\xi_t^m\}_{m,t}$ are i.i.d. random vectors and correspond to the noise in the observations. They are known to be zero mean, $\sigma^2$-sub Gaussian random vectors, i.e., they satisfy $\E[\exp(\lambda v^{\top} \xi_{t}^m)] \leq \exp(\lambda^2 \sigma^2/2d)$ for all $\lambda \in \R$, $t \in \{1,2,\dots, T\}, m \in \{1,2,\dots, M\}$ and unit vectors $v \in \R^d$.

Information exchange between the clients happens only through the server where each client can send its message to the server at the end of each time instant. Based on the messages received from the clients, the server can then choose to broadcast its message to the clients.
Both the uplink and the downlink channels have a finite capacity of $C$ bits per use, limiting the size of the messages. This model helps quantify communication cost to the bit level, a more challenging problem where the channel is assumed to have an infinite capacity in at least one direction.

A distributed learning algorithm consists of a decision strategy accompanied by a complementary communication strategy which decides when, how and what to communicate. The performance of an algorithm is measured in terms of the overall cumulative regret $R(T)$ and the cumulative communication cost $C(T)$ incurred by the algorithm. The overall cumulative regret is given by
\begin{align}
    R(T) = \sum_{m = 1}^M \sum_{t = 1}^T \left[ f(x_t^m) - f(x^*) \right],
\end{align}
where $x^* := \argmin_{x \in \cX} f(x)$. The communication cost $C(T)$ is measured using the uplink communication cost, denoted by $C_{\text{u}}(T)$, and the downlink communication cost, denoted by $C_{\text{d}}(T)$. The uplink communication cost, $C_{\text{u}}(T)$,  measures the \emph{number of bits} transferred by any client (on average) to the server throughout the learning horizon. Similarly, the downlink cost corresponds to the \emph{number of bits} broadcast by the server during the learning process.

The objective is to design a distributed learning algorithm that incurs order-optimal cumulative regret while minimizing the communication cost. In particular, the regret performance is measured against the benchmark of $\Omega(\log(MT))$, which is the optimal regret order in a centralized setting with a total of $MT$ observations. We would like to point out that this is a more challenging problem than that of achieving an order-optimal simple regret as achieving a low cumulative regret requires a finer control of exploration-exploitation trade-off.

\section{Algorithm Description}
\label{sec:algorithm}

In this section, we present our proposed algorithm, CEAL, that achieves order-optimal cumulative regret with a low communication overhead in a distributed optimization setup. After laying out the general structure of CEAL, we describe the design of its constituents. 

\subsection{Basic Structure of CEAL}

The framework underlying CEAL is inspired by that of the popular approach of Minibatch-SGD~\cite{Dekel2012}, wherein the same point is queried multiple times between two communication rounds. Specifically, CEAL proceeds in epochs which correspond to the time period between two communications. During an epoch $k \geq 1$, all the clients query the same point $x^{(k)}$ throughout the $t_k$ time instants of the epoch. At the end of the epoch, the clients compute the sample mean of the observed gradients, quantize it appropriately and send it to the server. The server combines the updates from all the clients and computes the next point as $x^{(k+1)} = x^{(k)} - \eta \hat{g}_k$ and broadcasts it to the clients after appropriate quantization. Here $\eta \in (0, 1/5\beta)$ is the step size and $\hat{g}_k$ denotes the noisy estimate of the gradient obtained by averaging the observations received from the clients. The design objective in CEAL is thus to choose the epoch lengths $\{t_k\}_{k \in \mathbb{N}}$ along with an associated communication protocol to ensure an order-optimal regret along with a low communication cost.

\subsection{The epoch lengths}

The epoch lengths play an important role in ensuring both a low regret and infrequent communication. In order to optimally design the epoch lengths, we first understand their role in achieving the desired performance guarantees. Using the definition of $x^{(k+1)}$, the $\beta$-smoothness of $f$, and the bound on $\eta$ we can show that
\begin{align}
    \E[f(x^{(k+1)})] & \leq \E[f(x^{(k)})] - \frac{\eta}{2} \E[\|\nabla f(x^{(k)})\|^2] + \frac{\eta}{2} \cdot \frac{\sigma^2}{Mt_k}. \nonumber
\end{align}
If one were to set $t_k$ to $2\sigma^2/(M\E[\| \nabla f(x^{(k)})\|^2])$, then using $\alpha$-smoothness of $f$, we can conclude that 
\begin{align}
    \Delta_{k+1} \leq (1 - \alpha \eta/2) \Delta_{k},
    \label{eqn:iterate_convergence}
\end{align}
where $\Delta_k := f(x^{(k)}) - f(x^*)$. In other words, the sub-optimality gap of the iterates decreases exponentially fast.

This choice of epoch lengths simultaneously offers the benefits of low regret and infrequent communication. Note that this choice of $t_k$ along with $\beta$-smoothness ensure that the regret incurred during each epoch is $\cO(1)$ as $t_k \cdot \sum_{m = 1}^M f(x^{(k)}) - f(x^*) = \cO(t_k \cdot M\| f(x^{(k)})\|^2) = \cO(1)$. Moreover, the relation $\| f(x^{(k)})\|^2 = \Theta(\Delta_k)$ along with eqn.\eqref{eqn:iterate_convergence} results exponentially increasing epoch lengths limiting the communication rounds to $\cO(\log T)$.

While such a choice of $t_k$ achieves the desired performance guarantees, its dependence on the knowledge of $\|\nabla f(x^{(k)})\|^2$, which is unknown at the beginning of the epoch, prevents one from setting the epoch length to this predetermined value. To overcome this hurdle, we design a novel norm estimation routine that adaptively estimates the norm of an unknown vector $y$ to within a required accuracy and terminates in $\cO(1/\|y\|^2)$ steps with high probability. If one were to integrate this routine with CEAL to estimate the norm of the unknown gradient $\nabla f(x^{(k)})$, not only would the server get an access to a more accurate estimate of the gradient, but also the epoch lengths would get adaptively and automatically chosen to the required order. Hence, CEAL uses the norm estimation routine described below to adaptively tune the epoch lengths.

\subsubsection{The Norm Estimation Routine}

The Norm Estimation Routine, as the name suggests, estimates the norm of unknown vector $y$ to a required accuracy using samples of the form $z = y + \xi$, where $\xi$ is a zero-mean noise satisfying the assumptions described in Section.~\ref{sec:problem_formulation}. This routine is also carried out in epochs with exponentially growing lengths. During each epoch, the clients compute a sample mean of the unknown vector and share it with the server. At the end of each epoch, a threshold-based termination test is employed at the server to determine whether the required estimation accuracy has been reached, which terminates the routine. A pseudo-code is provided in Algorithm~\ref{alg:norm_est}. The parameters $s_j$ and $\tau_j$ are specified later.

\begin{algorithm}
    \caption{\textsc{NormEst}}
    \label{alg:norm_est}
    \begin{algorithmic}[1]
            \STATE Set $j \leftarrow 1$
            \WHILE{\texttt{True}}
                \STATE For each client $m$, take $s_j$ samples and compute the sample mean $\hat{y}^{(m)}_j$ and send it to the server
                \STATE At the server, compute $\hat{y}_{j}^{(\textsc{serv})} =  \frac{1}{M} \sum_{m = 1}^M \hat{y}^{(m)}_j $
                \IF{$\tau_j \leq  \|\hat{y}_{j}^{(\textsc{serv})}\|_2/4 $}
                \STATE Server sends $\hat{y}_{j}^{(\textsc{serv})}$ to all clients
                \STATE \textbf{break}
                \ELSE
                \STATE $j \leftarrow j + 1$
                \ENDIF
            \ENDWHILE
    \end{algorithmic}
\end{algorithm}

\subsection{The Communication Strategy}
\label{sub:communication_strategy}

While the sequence $\{t_k\}_k$ controls the communication frequency between the clients and the server, the communication strategy of CEAL ensures that the messages are small to limit the number of bits transmitted over the channel. It consists of two components: quantizing the vector being transmitted and encoding the quantized vector to send it over the channel.

\noindent \textbf{Quantization}: CEAL quantizes a vector $y$ satisfying $\|y\| \leq r$ to an accuracy of $\varepsilon$ by quantizing each coordinate separately to an accuracy of $\varepsilon/\sqrt{d}$. Under this strategy, the interval $[-r,r]$ is first divided into $p(\varepsilon) = \lceil 2r \sqrt{d}/\varepsilon\rceil$ intervals of equal length and each coordinate is then quantized to one of $p(\varepsilon) + 1$ end points using the popular stochastic quantization routine. In particular, each coordinate $y_i$ is mapped to $Q_i$, where $Q_i$ takes either the value $c_{v-1}$ with probability $c_v - y$ or the value $c_v$ with the remaining probability. Here, $c_w := r \left( \dfrac{2w}{p(\varepsilon)} - 1 \right)$ for $w = 1,2, \dots, p(\varepsilon)$ and $v := \{w: c_{w-1} \leq y < c_w\}$. We use $Q(y, \varepsilon, r) = (Q_1, Q_2, \dots ,Q_d)^{\top}$ to denote the quantized version of the vector $y$.

\noindent \textbf{Encoding}: CEAL encodes this quantized vector to send it over a communication channel by encoding each coordinate of the quantized version, one by one, using the variable-length encoding strategy, unary coding. Each coordinate is encoded into a string of 1’s, whose length is equal to the absolute value of the coordinate, and is preceded by a sign bit, with $0$ corresponding to negative values and $1$ to positive ones. As an example, $-3$ and $4$ are encoded under this scheme as $0111$ and $11111$. 

Having specified all the components, we combine all of them together to provide a detailed description of CEAL in Algorithm~\ref{alg:cologne}. In the description, the initial point $x^{(1)}$ is chosen at random from the domain and assumed to be known to all the clients apriori. The epoch index $k$ corresponds to that of iterates and $j$ to that of the norm estimation routine. For clarity of notation, we refer to 

We set $s_j$ to $\lceil 40\sigma^2 \log(16Mj^2/\delta) 4^j/M \rceil$. The length of each epoch, $t_k$, is implicitly determined by $s_j$ as follows. If $\cJ_k$ denotes the set of all the different values of $j$ seen by the algorithm during the $k^{\text{th}}$ epoch, then $t_k = \sum_{j \in \cJ_k} s_j$. The parameters $\tau_j$ and $G_j$ are set to $3 \cdot 2^{-(j+1)}$ and $(4\sigma/\sqrt{s_j})(1 + \sqrt{\log(4Mj^2/\delta)/2d})$ respectively. They correspond to bounds on the estimation error at the server and any client respectively. $B_j$ corresponds to an upper bound on the gradient norm at the end of $j^{\text{th}}$ epoch and is set to $\min\{5\tau_{j-1}, 1\}$. Lastly, the resolution parameter sequences are set to $\gamma_j := \gamma_0 \sigma/\sqrt{s_j}$ and $\phi_j := \phi_0 \tau_j$ for constants $\gamma_0, \phi_0 \in (0,1)$, which are chosen based on communication requirements.

\begin{algorithm}
    \caption{Communication-Efficient Adaptive Learning (CEAL)}
    \label{alg:cologne}
    \begin{algorithmic}[1]
            \STATE \textbf{Input}: Initial point $x^{(1)}$, step size $\eta \in (0,1/\beta)$
            \STATE Set $k, j \leftarrow 1$
            \WHILE{time horizon is not reached}
                \STATE For each client $m$, take $s_j$ samples, compute the sample mean $\hat{g}^{(m)}_j(x^{(k)})$ and send $Q(\hat{g}^{(m)}_j(x^{(k)}), \gamma_j, G_j + B_j)$ to the server
                \STATE At the server, compute $\hat{g}^{(\textsc{serv})}_j(x^{(k)}) =  \frac{1}{M} \sum_{m = 1}^M \hat{g}^{(m)}_j(x^{(k)})$
                \IF{$\tau_j \leq \|\hat{g}^{(\textsc{serv})}_j(x^{(k)})\|_2/4 $}
                \STATE Server broadcasts $Q(\hat{g}^{(\textsc{serv})}_j(x^{(k)}), \phi_j, B_j + \tau_j) $ to all clients
                \STATE Clients update $x^{(k+1)} \leftarrow x^{(k)} - \eta Q(\hat{g}^{(\textsc{serv})}_j(x^{(k)}), \phi_j, B_j + \tau_j) $
                \STATE $k \leftarrow k + 1$, 
                \ELSE
                \STATE $j \leftarrow j + 1$
                \ENDIF
            \ENDWHILE
    \end{algorithmic}
\end{algorithm}

\section{Performance Analysis}
\label{sec:analysis}

In this section, we characterize the performance of CEAL in terms of cumulative regret and the communication cost it incurs. We begin with bounding the regret incurred by CEAL.

\subsection{Regret analysis}

The following theorem characterizes the cumulative regret incurred by CEAL.

\begin{theorem}
    Consider a distributed learning setup as described in Sec.~\ref{sec:problem_formulation}. If CEAL is run with parameters described in Sec.\ref{sec:algorithm} then the cumulative regret incurred by CEAL is bounded by $\cO(\log (MT) \log(M/\delta))$ with probability at least $1- \delta$.
    \label{theorem:regret}
\end{theorem}

Theorem~\ref{theorem:regret} establishes the regret performance of CEAL. Note that it matches with the lower bound for any algorithm in a centralized setting with $MT$ queries implying that the regret performance of CEAL is indeed order-optimal. 

\begin{proof}

To bound the regret incurred by CEAL, note that in each epoch $k$, each client queries the point $x^{(k)}$ for a total of $t_k$ times. Consequently, the regret incurred by CEAL can be written as
\begin{align}
    R(T) & = \sum_{m = 1}^M \sum_{k = 1}^K  (f(x^{(k)}) - f(x^*)) \cdot t_k \nonumber \\
    & \leq 2\beta M \sum_{k = 1}^K \| \nabla f(x^{(k)} \|^2 \cdot t_k,\label{eqn:regret_first_step}
\end{align}
where $K$ is the (random) number of epochs carried out during the execution of the algorithm. The following lemma provides a bound on the length of the $k^{\text{th}}$ iteration, $t_k$.

\begin{lemma}
    Suppose CEAL is run with parameters described in Sec.\ref{sec:algorithm}. If it queries a point $x^{(k)}$ during epoch $k$, then the length of epoch $k$, $t_k$, as defined in Sec.~\ref{sec:algorithm} satisfies $\Theta(1/(M\|\nabla f(x^{(k)})\|)^{2})$ with probability at least $1 - \delta$.
    \label{lemma:t_k}
\end{lemma}

In addition to the above bound, the iteration lengths also satisfy the constraint $\sum_{k = 1}^K t_k \leq T$ as each client cannot issue more than $T$ gradient queries. Consequently, the upper bound on regret is the value of the following constrained maximization problem which is obtained by using bounds on $t_k$ described in Lemma~\ref{lemma:t_k}. 
\begin{align*}
    \max & \ \ \ 2 \beta \sum_{k = 1}^K \bigg[ 4320\sigma^2 \cdot \log\left(\frac{16M}{\delta} \log^2 \left( \frac{9}{2\|\vartheta_k \|}\right) \right)  + \\
    & \hspace{7em} M \| \vartheta_k\|^2 \left( \log \left( \frac{9}{2 \| \vartheta_k \|} + 1 \right) \right) \bigg] \\
    \text{s.t.} &  \sum_{k = 1}^K \bigg[ \frac{4320\sigma^2}{M \|\vartheta_k\|^2} \cdot \log\left(\frac{16M}{\delta} \log^2 \left( \frac{9}{2\|\vartheta_k\|}\right) \right)  + \\
    & \hspace{8em} \left( \log \left( \frac{9}{2 \| \vartheta_k \|} + 1 \right) \right) \bigg] \leq T,
\end{align*}
where $\vartheta_k := \nabla f(x^{(k)})$. On applying the method of Lagrange Multipliers, one can immediately conclude that value of the aforementioned constrained problem is $\cO(K \log(M/\delta))$, i.e., proportional to the number of iterations carried out during the algorithm. The following lemma provides high probability bound on the number of epochs during a execution of CEAL.

\begin{lemma}
    If CEAL is run with parameters described in Sec.\ref{sec:algorithm}, then both the total number of epochs during its run is bounded by $\cO(\log(MT))$ with probability at least $1 - \delta$.
    \label{lemma:number_of_communications}
\end{lemma}

The theorem now follows immediately by noting that $R(T)$ is $\cO(K)$ and invoking the above lemma. Since both Lemma~\ref{lemma:t_k} and~\ref{lemma:number_of_communications} provide high probability bounds, the resultant bound on the regret incurred by CEAL also holds with probability at least $1 - \delta$.
    
\end{proof}

Please refer to Appendix~\ref{sec:proofs} for proofs of Lemmas~\ref{lemma:t_k} and~\ref{lemma:number_of_communications}.


\subsection{Communication Cost}

The following theorem establishes the performance on the communication cost incurred by CEAL.

\begin{theorem}
    Consider a distributed learning setup as described in Sec.~\ref{sec:problem_formulation}. If CEAL is run with parameters described in Sec.\ref{sec:algorithm} then the uplink and the downlink communication costs (in bits) incurred by CEAL, i.e. $C_{\text{u}}(T)$ and $C_{\text{d}}(T)$, satisfy $\cO\left( \dfrac{d}{\gamma_0} \log (MT)\right)$ and $\cO\left( \dfrac{d}{\phi_0} \log(MT)\right)$ respectively.
    \label{theorem:comm_cost}
\end{theorem}

\begin{proof}

The main ingredient in the proof of the above theorem is the following lemma that bounds the size of any message exchanged during the algorithm.

\begin{lemma}
    If CEAL is run with parameters described in Sec.\ref{sec:algorithm}, then under the communication scheme described in Sec.~\ref{sub:communication_strategy}, the size of any message exchanged between a client and the server is bounded by $\cO(d)$ bits.
    \label{lemma:channel_capacity}
\end{lemma}

From the above lemma, we know that each message transmitted both on the uplink and downlink channel is $\cO(d)$ bits. Since messages are exchanged only once during each epoch, i.e., when the epoch index $k$ is updated (See lines $7-9$ in Alg.~\ref{alg:cologne}), the total communication cost on both uplink and downlink channel is $\cO(d)$ times the number of epochs. The statement of the theorem now immediately follows by invoking Lemma~\ref{lemma:number_of_communications}.
\end{proof}


Please refer to Appendix~\ref{sec:proofs} for a detailed proof of Lemma~\ref{lemma:channel_capacity}.

\section{Empirical Studies}

\begin{figure*}[!ht]
\centering
\subfloat[Synthetic Dataset]{\label{fig:synthetic_regret_plot}\centering \includegraphics[scale = 0.375]{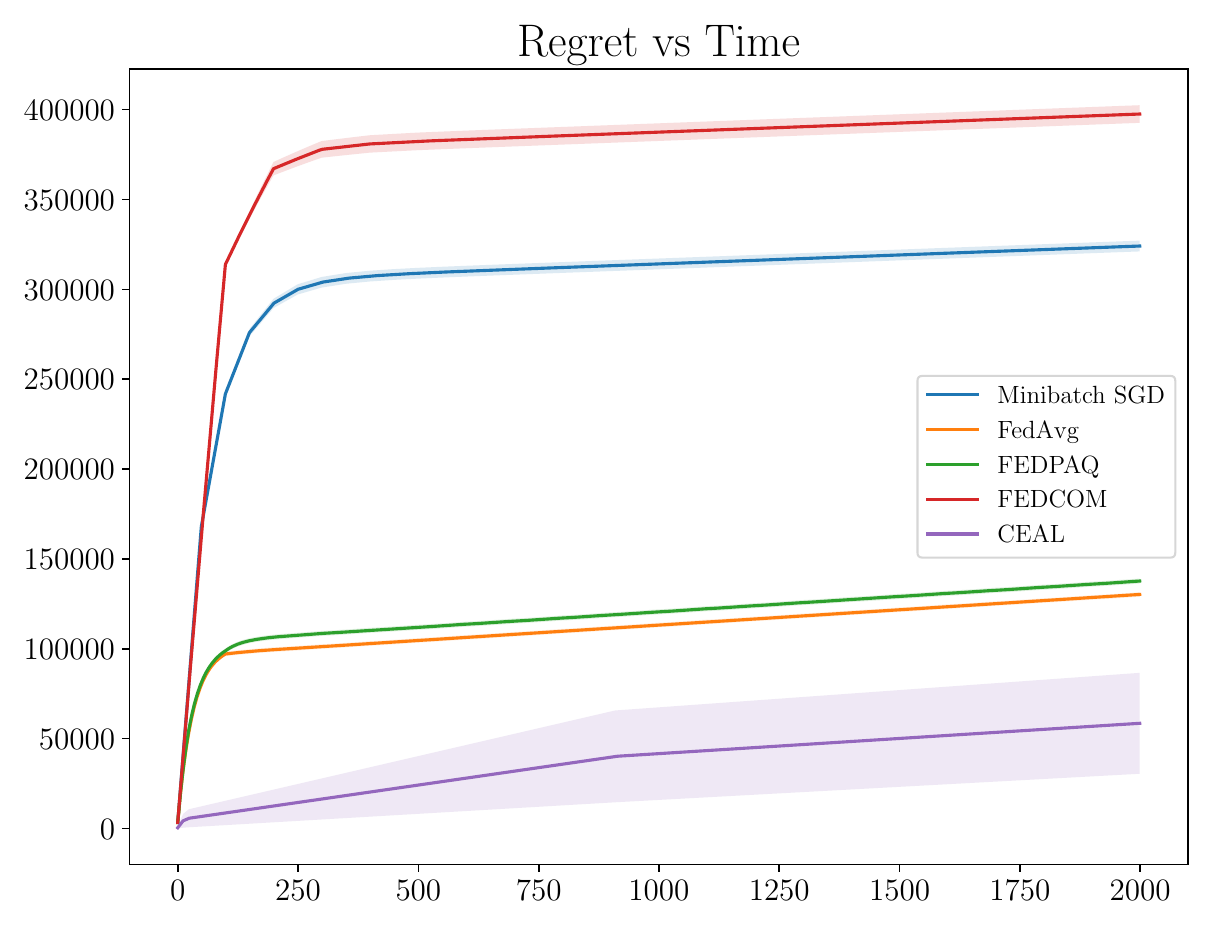}}
~
\subfloat[MNIST Dataset]{\label{fig:mnist_regret_plot}\centering \includegraphics[scale = 0.35]{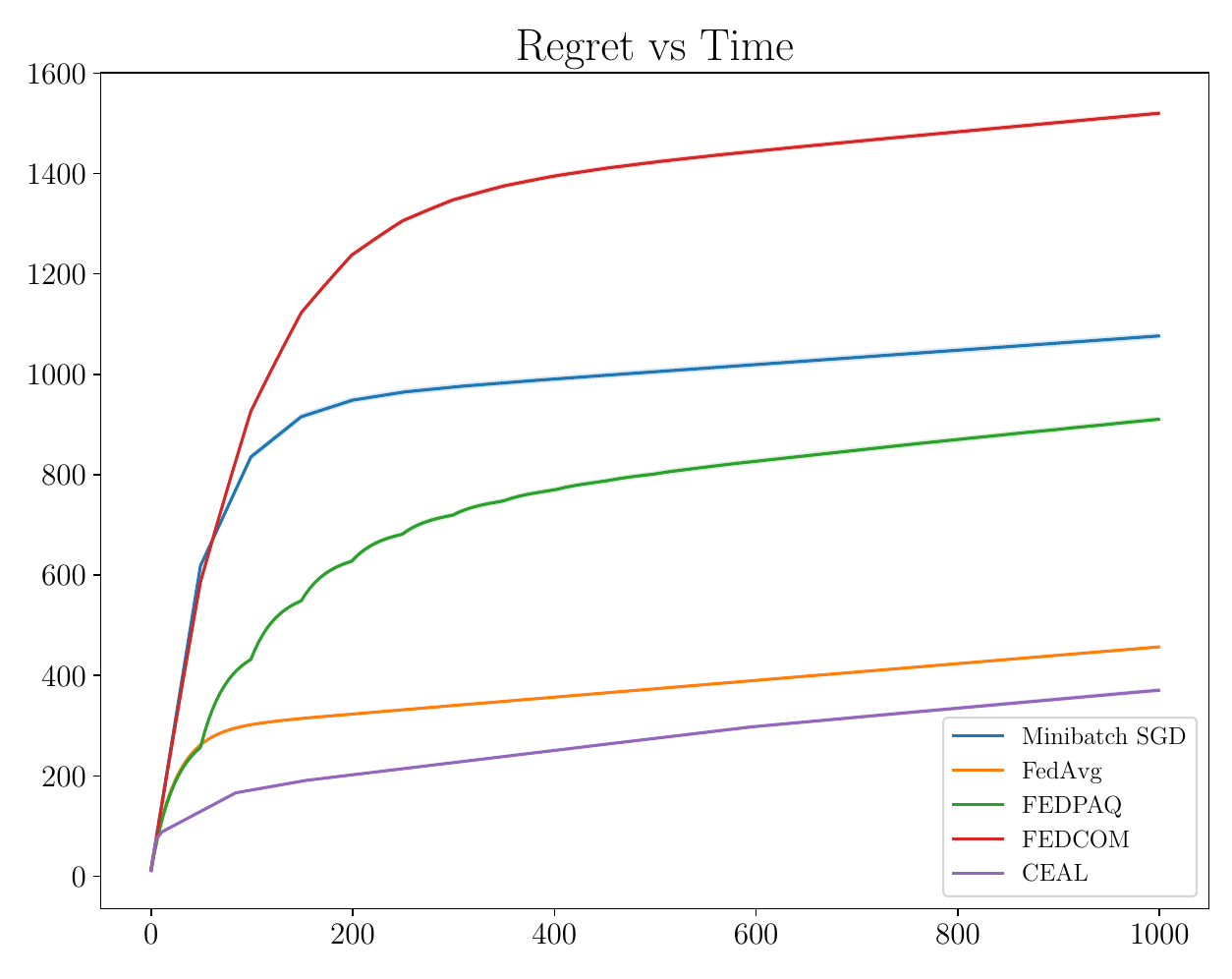}}
~
\caption{Cumulative Regret vs Time for different algorithms for on (a) Synthetic dataset and (b) MNIST dataset. The bold line represents the mean obtained over $10$ Monte Carlo runs and the shaded region represents the region of error bars corresponding to one standard deviation. The error bars for all algorithms in on MNIST are very small (approximately $\pm 10$).}
\label{fig:plots}
\end{figure*}

\begin{center}
\begin{table*}
\centering
\small
    \begin{tabular}{ccccc}
    \toprule
    & \multicolumn{2}{c}{Uplink Cost} & \multicolumn{2}{c}{Downlink Cost}  \\
    \midrule
      & Synthetic Dataset & MNIST & Synthetic Dataset & MNIST \\
      \cmidrule{2-5}
     Minibatch SGD & $38400$  & $5.02 \times 10^6$ & $38400$ & $5.02 \times 10^6$  \\ 
     FedAvg & $19200$ & $5.02 \times 10^6$ & $19200$ & $5.02 \times 10^6$  \\
     FedPAQ & $2440$ & $0.63 \times 10^6$ & $19200$ & $5.02 \times 10^6$ \\
     FedCOM & $2440$ & $0.63 \times 10^6$& $19200$ & $5.02 \times 10^6$ \\
     CEAL & $263.3$ & $0.11 \times 10^6$& $288.6$ & $0.26 \times 10^6$ \\ 
     \bottomrule
    \end{tabular}
    \caption{Communication cost (in bits) for various algorithms on different datasets. Reported values are obtained after averaging over $10$ Monte Carlo runs.}
    \label{table:comm_costs}
\end{table*}
\end{center}

In this section, we provide numerical experiments comparing the performance of CEAL with that of several baselines, namely, Minibatch-SGD~\cite{Woodworth2020}, FedAvg~\cite{BrendanMcMahan2017}, FedPAQ~\cite{Reisizadeh2020fedpaq} and FedCOM~\cite{Haddadpour2021federated}.

We first describe the datasets and experimental settings which is followed by a discussion on the results.

\subsection{Datasets}

We perform empirical studies on both synthetic and real-world datasets. For the synthetic dataset, we consider the problem of linear regression where the loss (objective) function is given by $f(\theta) = \frac{1}{N}\sum_{i = 1}^N (y_i - X_i^{\top}\theta)^2$, where $\theta \in \R^{30}$. The covariates $\{X_i\}_{i = 1}^N$ are drawn from zero mean normal distribution and are normalized and scaled to have norm of $100$. The responses $\{y_i\}_{i =1}^N$ are generated as $y_i = X_i^T \theta^* + \varepsilon_i$, where $\varepsilon_i$'s are independent and identically distributed as $\cN(0,1)$ and $\theta^*$ is the true unknown set of regression coefficients and is drawn uniformly at random from the surface of the unit sphere. The number of data points are set to $N = 2000$ and are distributed uniformly across the $M = 10$ devices, with each getting $200$ data points.

We also consider the problem of regularized logistic regression on MNIST dataset. From the original training dataset of $60,000$ images, we consider a subset of $50,000$ images, with $5,000$ images corresponding to each digit. This dataset is uniformly distributed across all $M = 10$ clients, resulting in $5,000$ data points for each client. The images are normalized to ensure that the pixel values lie in $[0,1]$. We consider the standard loss function for multinomial logistic regression given by
\begin{align*}
    f(W) & = \frac{1}{N} \sum_{i = 1}^N \bigg[\sum_{k = 0}^9 \1\{Y_i = k\} X_i^T W + \\
    &~~~~~~~~~\log\left(\sum_{k = 0}^9 \exp(-X_i^T W_k) \right) \bigg] + \mu \|W\|_F^2
\end{align*}
Here $W \in R^{784 \times 10}$ is a weight matrix for classification, $X_i \in \R^{784}$ is the vectorized training data, $Y_i$ is the corresponding label, $k \in \{0,1,\dots, 9\}$ denotes the class, $N = 50,000$ is the number of data points, $\mu > 0$ is the regularization parameter and $\| \cdot \|_F$ denotes the Frobenious norm.

\subsection{Experimental Settings}

For both the synthetic and real dataset, we consider a Federated Learning problem with $M = 10$ devices connected to a central server. All the algorithms are run for a time horizon of $T = 2000$ steps for the syntethic dataset and $T = 1000$ for the MNIST dataset. The gradient is computed using a randomly chosen minibatch for each client, based on the each client's data. For the synthetic and MNIST datasets, the minibatch size is set to $1$ and $25$ respectively. The learning rate for all the algorithms is optimized using a grid search. For the synthetic dataset, the learning rate of Minibatch-SGD, FedAvg, FedPAQ, FedCOM\footnote{The global learning rate was set to $10$ in both the experiments.} and CEAL were set to $1, 0.1, 0.1, 0.002$ and $2$ respectively. Similarly, for the MNIST dataset the learning rates (in the same order) were set to $0.2, 0.01, 0.01, 0.0005$ and $0.3$ respectively. For the synthetic dataset, the number of local steps was set to $100$ for FedAvg, FedPAQ and FedCOM and $50$ for Minibatch SGD. For the real dataset, the number of local steps was set to $50$ for all algorithms. The regularization parameter was set to $0.5$.

We assume $32$ bit representation of floats in order to calculate the communication cost for algorithms that do not employ quantization. The number of quantization levels for FEDPAQ and FEDCOM is set to $3$ for synthetic dataset and $5$ for real dataset. We report the cumulative regret and the communication cost (both uplink and downlink), measured in bits for all the algorithms, averaged over $10$ Monte Carlo runs. Recall that the uplink and downlink cost were defined to the number of bits transmitted by a client (on average) to the server and the those broadcast by the server to the clients throughout the entire learning process.

\subsection{Results}

We plot the overall cumulative regret incurred by different algorithms for the experiment with synthetic dataset in Fig.~\ref{fig:synthetic_regret_plot} and that for the real dataset in Fig.~\ref{fig:mnist_regret_plot}. We tabulate the communication costs incurred by different algorithms across different experiments in Table~\ref{table:comm_costs}. As it is evident from the plots, our proposed algorithm achieves a smaller cumulative regret than the standard, commonly used algorithms for Federated Learning for both the tasks. Moreover, this improved performance in regret is achieved at a very low communication cost. Specifically, for synthetic and real dataset the uplink communication cost incurred by CEAL is about $10\%$ and $16\%$ respectively of that incurred by FedPAQ and FedCOM algorithms, both of which popular Federated Learning algorithms that employ quantization to reduce communication. This factor reduces to $1-2\%$ when compared against classical algorithms like FedAvg that do not employ quantization. CEAL also offers significant reduction in downlink costs as it incurs no more than $2\%$ of the downlink cost incurred by all other algorithms. This can significantly reduce the download costs for local training devices and consequently reduce the infrastructure requirements for local participating devices.

\section{Conclusion}
\label{sec:conclusion}

We proposed a new algorithm for distributed convex optimization called CEAL that incurs order-optimal cumulative regret of $\cO(\log(MT))$ along with a communication cost of $\cO(d \log(MT))$ bits. CEAL is characterized by its adaptive learning through its novel norm estimation routine that allows it to achieve desirable performance in terms of the more holistic measures of regret and communication cost considered in this work. An interesting future direction is to investigate the lower bounds on this holistic communication cost and compare them to the upper bounds established in this work. Another interesting direction would to extend CEAL to general convex functions.

\bibliographystyle{IEEEtran}
\bibliography{references}

\appendix

In Appendix~\ref{sec:proof_norm_estimation_routine}, we first analyze the performance of the Norm Estimation Routine, which is an important component in the design of CEAL.  We then use the results obtained in Appendix~\ref{sec:proof_norm_estimation_routine} to prove the main result of Theorem~\ref{theorem:regret} in Appendix~\ref{sec:proofs}.

\section{Norm Estimation Routine}
\label{sec:proof_norm_estimation_routine}

Consider the task of estimating the norm of a vector $y$ with $\|y\| \leq 1$ using the Norm Estimation Routine. For completeness, we again provide the pseudo-code for the routine in Algorithm~\ref{alg:norm_est_appendix}. Note that this description also includes the quantization step, which is assumed to be carried out using the quantization process described in Section~\ref{sub:communication_strategy}.

\begin{algorithm}
    \caption{\textsc{NormEst}}
    \label{alg:norm_est_appendix}
    \begin{algorithmic}[1]
            \STATE Set $j \leftarrow 1$
            \WHILE{\texttt{True}}
                \STATE For each client $m$, take $s_j$ samples and compute the sample mean $\hat{y}^{(m)}_j$, quantize it to $Q(\hat{y}^{(m)}_j)$ and send it to the server
                \STATE At the server, compute $\hat{y}_{j}^{(\textsc{serv})} =  \frac{1}{M} \sum_{m = 1}^M Q(\hat{y}^{(m)}_j)$
                \IF{$\tau_j \leq  \|\hat{y}_{j}^{(\textsc{serv})}\|_2/4 $}
                \STATE Server sends $\hat{y}_{j}^{(\textsc{serv})}$ to all clients
                \STATE \textbf{break}
                \ELSE
                \STATE $j \leftarrow j + 1$
                \ENDIF
            \ENDWHILE
    \end{algorithmic}
\end{algorithm}

The estimates obtained during the Norm estimation phase satisfy the following lemma.

\begin{lemma}
For any epoch $j$ during the Norm Estimation Routine, the estimate at the server, $\hat{y}_{j}^{(\textsc{serv})}$, satisfies,
\begin{align*}
    \|\hat{y}_{j}^{(\textsc{serv})} - y\| \leq  \tau_j,
\end{align*}
with probability at least $1 - \delta$.
\label{lemma:estimate_err}
\end{lemma}

\begin{proof}
Let $\eta_j^{(m)} \in \R^d$ denote the quantization noise added by $m^{\text{th}}$ client during the $j^{\text{th}}$ epoch, i.e., the quantized version received by the server can be written as $Q(\hat{y}^{(m)}_j) = \hat{y}^{(m)}_j + \eta_j^{(j)}$. Since each coordinate is quantized independently, each coordinate of $\eta_j^{(m)}$ is an independent zero mean sub-Gaussian random variable with variance proxy $\gamma_j^2/4d$. At the end of the $j^{\text{th}}$ epoch, we have

\begin{align*}
    \|\hat{y}_{j}^{(\textsc{serv})} - y\| & = \left\|\frac{1}{M} \sum_{m = 1}^M Q(\hat{y}^{(m)}_j) - y\right\| \\
    & = \left\|   \frac{1}{M} \sum_{m = 1}^M (\hat{y}^{(m)}_j + \eta_j^{(m)}) - y \right\| \\
    & \leq \left\|  \frac{1}{M} \sum_{m = 1}^M\hat{y}^{(m)}_j  - y \right\|   + \left\| \frac{1}{M} \sum_{m = 1}^M \eta_j^{(m)}  \right\|.
\end{align*}

Let $\mathcal{E}_j$ denote the event where the observations satisfy the following two inequalities:
\begin{align*}
    \left\| \frac{1}{M} \sum_{m = 1}^M \eta_j^{(m)}  \right\| & \leq \frac{2\gamma_j}{\sqrt{M}}\left(1 + \sqrt{\frac{1}{2d}\log\left(\frac{4j^2}{\delta}\right)} \right) \\
    \left\|  \frac{1}{M} \sum_{m = 1}^M\hat{y}^{(m)}_j - y \right\|   & \leq \frac{4\sigma}{\sqrt{Ms_j}}\left(1 +  \sqrt{\frac{1}{2d}\log\left(\frac{4j^2}{\delta}\right)} \right).
\end{align*}
Similarly, define $\mathcal{E} = \cap_{j \geq 1} \mathcal{E}_j$. Using the concentration of sub-Gaussian random vectors~\cite{Jin2019, Salgia2022Bandits}, we obtain that $\mathcal{E}_j$ holds with probability at least $1 - \delta/(2j^2)$. Consequently, an application of the union bound yields $\Pr(\mathcal{E}) \geq 1 - \delta \sum_{j = 1}^{\infty}(2j^2)^{-1} \geq 1 - \delta$.

On plugging the values of $s_j$ and $\gamma_j$ and conditioning on the event $\mathcal{E}$, we obtain that
\begin{align*}
    \|\hat{y}_{j}^{(\textsc{serv})} - y\| \leq {2^{-j}} + {2^{-(j+1)}}  = 3 \cdot 2^{-(j+1)} = \tau_j,
\end{align*}
holds for all $j \geq 1$.

\end{proof}

\section{Proofs of Helper Lemmas}
\label{sec:proofs}

We begin by proving Lemmas~\ref{lemma:t_k} and~\ref{lemma:number_of_communications} which we then use to establish the result in Theorem~\ref{theorem:regret}. 

\subsection{Proof of Lemma~\ref{lemma:t_k}}
\label{proof:lemma_t_k}

Consider the $k^{\text{th}}$ iterate, $x^{(k)}$. Let $g(x^{(k)}) := \nabla f(x^{(k)})$ denote the true gradient at $x^{(k)}$. Recall that the $\hat{g}_j^{(\textsc{serv})}(x^{(k)})$ denotes the estimate of the gradient at the server at the end of the $j^{\text{th}}$ epoch. For brevity of notation, we drop the argument $x^{(k)}$ throughout this proof.

Under the event $\cE$ as defined in Proof of Lemma~\ref{lemma:estimate_err}, the Norm Estimation Routine at this point will terminate at the end of epoch $j_0$, where $j_0 := \min \{ j \in \N: 4\tau_j \leq \| \hat{g}_{j}^{(\textsc{serv})} \| \}$. We note that for all $j$ for which the inequality $4\tau_j \leq \| \hat{g}_{j}^{(\textsc{serv})}\|$ holds, we also have the relation
\begin{align*}
     \|\hat{g}_{j}^{(\textsc{serv})} - g\| &\leq \tau_j \leq \frac{1}{4} \cdot \| \hat{g}_{j}^{(\textsc{serv})} \|  \\
     \implies \|g\| &\in \left[ \frac{3}{4} \| \hat{g}_{j}^{(\textsc{serv})} \|, \frac{5}{4} \| \hat{g}_{j}^{(\textsc{serv})} \| \right] \\
     \implies \| \hat{g}_{j}^{(\textsc{serv})} \| &\in \left[ \frac{4}{5} \|g\|, \frac{4}{3} \|g \| \right].
\end{align*}
Consequently, we also have $\tau_j \leq \frac{1}{3} \cdot \| g \|$ which implies that $j \geq \log_2(9/2\|g\|)$. Since $j_0$ is the smallest natural number satisfying this relation, $\displaystyle j_0 = \lceil\log_2(9/2\|g\|) \rceil$.

Thus, under the event $\cE$, the Norm Estimation Routine terminates at the end of epoch $j_0$. Consequently, we can bound $t_k$ as
\begin{align*}
    t_k & \leq \sum_{j = 1}^{j_0} s_j \leq \sum_{j = 1}^{j_0} \left[ \frac{40\sigma^2}{M} \cdot \log(16Mj^2/\delta) \cdot 4^j + 1 \right] \\
    & \leq \frac{160\sigma^2}{3M} \cdot \log(16Mj_0^2/\delta) \cdot 4^{j_0} + {j_0}.
\end{align*}
On plugging in the value of $j_0$, we obtain that $t_k$ satisfies $\cO \left( \frac{1}{M \|\nabla f(x^{(k)})\|^2} \right)$. 

Similarly, we also can obtain a lower bound on $t_k$. Note that $t_k \geq s_{j_0}$. Substituting the expression for $s_j$ and the value of $j_0$ yields us that $t_k$ satisfies $\Omega\left(\frac{1}{M \|\nabla f(x^{(k)})\|^2}\right)$.

\subsection{Proof of Lemma~\ref{lemma:number_of_communications}}
\label{proof:lemma_number_of_communications}

To establish an upper bound on the number of epochs (and communication rounds), we first establish a convergence rate of the iterates. In particular, we show that conditioned on the event $\cE$, the magnitude of the gradient across iterates decreases exponentially fast, resulting in a logarithmic number of communication rounds. We use the same notation as described in Appendix.~\ref{proof:lemma_t_k}. In addition, we use $\hat{[g]}_j^{(\textsc{serv})}(x^{(k)})$ to denote the quantized version of $\hat{g}_j^{(\textsc{serv})}(x^{(k)})$, i.e., $Q(\hat{g}^{(\textsc{serv})}_j(x^{(k)}), \phi_j, B_j + \tau_j)$, which is sent to the clients. Once again, for brevity of notation, we drop the argument $x^{(k)}$ from the gradient terms throughout this proof.

The quantization strategy used in CEAL guarantees that $\|\hat{g}_j^{(\textsc{serv})} - \hat{[g]}_j^{(\textsc{serv})}\| \leq \tau_j$. Consequently, $\|\hat{[g]}_j^{(\textsc{serv})} - g\| \leq 2\tau_j \leq \frac{2}{3} \cdot \| g \|$ which leads to the bounds $\ip{g}{ \hat{[g]}_j^{(\textsc{serv})}} \geq \|g\|^2/3$ and $\|\hat{[g]}_j^{(\textsc{serv})}\| \leq 5\| g \|/3$. Using the $\beta$-smoothness of $f$, we have,
\begin{align*}
    f(x^{(k+1)}) & = f(x^{(k)} - \eta \hat{[g]}_j^{(\textsc{serv})}(x^{(k)})) \\
    & \leq f(x^{(k)}) - \eta \ip{g}{\hat{[g]}_j^{(\textsc{serv})}} + \frac{\eta^2 \beta}{2} \| \hat{[g]}_j^{(\textsc{serv})} \|^2 \\
    & \leq f(x^{(k)}) - \frac{\eta}{3} \|\nabla f(x^{(k)})\|^2  + \frac{25\eta^2 \beta}{18} \|\nabla f(x^{(k)})\|^2 \\
    & \leq f(x^{(k)}) - \frac{\eta}{3} \|\nabla f(x^{(k)})\|^2  + \frac{5\eta}{18} \|\nabla f(x^{(k)})\|^2 \\
    & \leq f(x^{(k)}) - \frac{\eta}{18} \|\nabla f(x^{(k)})\|^2 ,
\end{align*}
where the fourth step uses the relation $\eta \leq 1/(5\beta)$. Using the $\alpha$-strong convexity of $f$ we obtain,
\begin{align*}
    f(x^{(k+1)}) - f(x^*) \leq \left( 1- \frac{\alpha \eta}{9} \right) (f(x^{(k)}) - f(x^*)).
\end{align*}
Consequently, the sub-optimality gap after $k$ iterates is given by
\begin{align*}
    f(x^{(k)}) - f(x^*) \leq (1 - \alpha \eta/9)^{k-1} (f(x^{(1)}) - f(x^*)).
\end{align*}
Using $\beta$-smoothness of $f$, we can translate this bound to a bound on the magnitude of the gradient, i.e., $\|\nabla f(x^{(k)})\|^2 
 \leq C (1 - \alpha \eta/9)^{k-1}$ for some constant $C > 0$. 
 
To obtain the bound on the number of epochs, note that the lower bound on $t_k$ obtained in Sec.~\ref{proof:lemma_t_k} suggests that if the inequality $\|\nabla f(x^{(K)})\|^2 \leq C'/MT$ holds for some iterate $x^{(K)}$ and a constant $C' > 0$, independent of $M$ and $T$, then $t_K > T$, implying the algorithm terminates by epoch $K$. Using the exponential convergence of the gradient magnitude across iterates, we can conclude that such an epoch index $K$ satisfies $\cO(\log(MT))$. 

\subsection{Proof of Lemma~\ref{lemma:channel_capacity}}
\label{proof:lemma_channel_capacity}

Let $y \in \R^d$ be a vector with $\|y\| \leq r$ and $Q(y) = (Q_1, Q_2, \dots, Q_d)$ denote its quantized version up to a precision of $\varepsilon$, as carried out in CEAL. Firstly, we have $\|y - Q(y) \| \leq \varepsilon \implies \|Q(y)\| \leq r + \varepsilon$. From the definition of $Q(y)$, we have,
\begin{align*}
    \sum_{i = 1}^d Q_i^2 \left( \frac{2r}{p(\varepsilon)}\right)^2 & \leq (r + \varepsilon)^2 \\
    \implies \sum_{i = 1}^d Q_i^2 & \leq \left(\frac{(r + \varepsilon) p(\varepsilon)}{2r} \right)^2 \\
    & \leq 4d \left(\frac{r}{\varepsilon} + 1 \right)^2.
\end{align*}
Recall that $p(\varepsilon) = \lceil 2r \sqrt{d}/\varepsilon \rceil$ denotes the number of intervals along each coordinate in the quantization process. Consequently, 
\begin{align*}
    \sum_{i = 1}^d |Q_i| \leq \sqrt{d\sum_{i = 1}^d Q_i^2} \leq 2d \left(\frac{r}{\varepsilon} + 1 \right).
\end{align*}

It can be noted that under the encoding scheme used in CEAL, the message size in bits in CEAL is given by $d + \sum_{i = 1}^d |Q_i|$, where $d$ is added to account for the sign bit of each coordinate. Consequently, the message size is bounded by $d (1 + 2(r/\varepsilon + 1))$. The message size for both uplink and downlink communication is obtained by plugging in the appropriate value $r$ and $\varepsilon$.

Let us first consider the uplink communication. In epoch $j$, $r$ corresponds to $G_j + B_j$ and $\varepsilon$ to $\gamma_j$. On plugging in the prescribed values of the above parameters, we obtain that their ratio is $C/\gamma_0$, where $C$ is a constant independent of $d, M$ and $T$, resulting in a message size of $\cO(d)$ bits. Similarly, for the downlink communication, the ratio $(B_j + \tau_j)/\phi_j \leq C'/\phi_0$ is also a constant leading to a message size of $\cO(d)$ bits. The constants $\gamma_0$ and $\phi_0$ appearing in the denominator in the statement of Theorem~\ref{theorem:comm_cost} are a direct consequence of leading constant obtained above.

\end{document}